\documentclass{article}

\usepackage{arxiv}

\usepackage[utf8]{inputenc} 
\usepackage[T1]{fontenc}    
\usepackage[colorlinks,citecolor=blue,urlcolor=blue,linkcolor=magenta]{hyperref}
\usepackage{url}            
\usepackage{booktabs}       
\usepackage{amsfonts}       
\usepackage{nicefrac}       
\usepackage{microtype}      
\usepackage{natbib}
\usepackage{graphicx}
\usepackage{amsmath}
\usepackage{amssymb}
\usepackage{amsthm}
\usepackage{algorithm}
\usepackage{algorithmic}
\usepackage{subcaption}

\usepackage{multirow}

\usepackage{color}
\usepackage{mathrsfs}
\graphicspath{ {./images/} }
\usepackage{mathtools}

\newcommand{\X}{\mathbf{X}}
\newcommand{\lrc}[1]{\left\{#1 \right\}}
\renewcommand*{\P}{\mathbb{P}}

\newcommand{\black}[1]{\textcolor{black}{#1}}

\newtheorem{definition}{Definition}[section]
\newtheorem{theorem}{Theorem}[section]
\newtheorem{example}{Example}[section]
\newtheorem{lemma}{Lemma}[section]
\newtheorem{remark}{Remark}[section]

\title{SLDP: Semi-Local Differential Privacy for Density-Adaptive Analytics}

\author{
Alexey Kroshnin \\
 Weierstrass Institute for\\ Applied Analysis and Stochastic\\
  \texttt{kroshnin@wias-berlin.de} \\
   \And
 Alexandra Suvorikova \\
  Weierstrass Institute for\\ Applied Analysis and Stochastic\\
  \texttt{suvorikova@wias-berlin.de} \\
}

\begin{document}
\maketitle
\begin{abstract}
Density-adaptive domain discretization is essential for high-utility privacy-preserving analytics but remains challenging under Local Differential Privacy (LDP) due to the privacy-budget costs associated with iterative refinement. We propose a novel framework---Semi-Local Differential Privacy (SLDP)---that assigns a privacy region to each user based on local density and defines adjacency by the potential movement of a point within its privacy region. We present an interactive 
$(\varepsilon, \delta)$-SLDP protocol, orchestrated by an honest-but-curious server over a public channel, to estimate these regions privately. Crucially, our framework decouples the privacy cost from the number of refinement iterations, allowing for high-resolution grids without additional privacy budget cost. We experimentally demonstrate the framework's effectiveness on estimation tasks across synthetic and real-world datasets.
\end{abstract}

%

\section{Introduction}
Many privacy-sensitive tasks---from spatial range queries \cite{tire2024answering,  alptekin2025hierarchical} and density estimation \cite{butucea2020local, kroll2021on} to classification \cite{berrett2019classification, fletcher2019decision} and hypothesis testing \cite{narayanan2022private, lamweil2022minimax}---require capturing data geometry. Specifically, dense regions benefit from fine discretization, whereas sparse regions necessitate aggregation to balance utility and privacy. The Differential Privacy (DP) \cite{dwork2006differential} framework protects individuals by ensuring that the analysis output remains stable even when a single record is modified. Centralized algorithms achieve high utility by accessing raw data to capture geometry \cite{zeighami2021neural,liu2024differentially, zhang2016privtree}, but they necessitate sharing unperturbed data with a third party (trusted curator).

To avoid a trusted curator, Local Differential Privacy (LDP) lets each user randomize their data on-device before collection \cite{kasiviswanathan2011can, xiong2020comprehensive}. This eliminates centralized access to raw data, but often incurs a loss in utility \cite{kairouz2014extremal}.
Interactive LDP mechanism \cite{joseph2019role} can adapt to data geomrtry over multiple rounds, but requires additional communication and, when users contribute repeatedly, allocating the privacy budget across rounds via composition \cite{kairouz2015composition}; see, e.g., \cite{alptekin2025hierarchical, balioglu2024grid, neto2026alog}.

While metric notions such as geo-indistinguishability (geo-DP) incorporate the geometry of the input domain, they typically use a fixed metric and a global privacy scale \cite{andres2013geo, chatzikokolakis2013broadening}.
However, standard Geo-DP mechanisms apply isotropic noise independent of the data distribution.

Together, these observations raise a natural question: \emph{can we obtain a trusted-curator-free mechanism that adapts to local data geometry and improves utility over LDP algorithms?}
We answer this in a new framework, \emph{Semi-Local Differential Privacy (SLDP)}. 

SLDP restricts the privacy guarantee by constraining the user’s indistinguishability set to a data-dependent privacy region. To implement this, we propose an interactive protocol that privately constructs a $k$-anonymous partition of the domain through iterative refinement. Under this framework, a user’s privacy region is defined as the specific leaf cell containing their record. By construction, our protocol ensures each leaf cell maintains a population of at least $k$ users, enabling a density-aware discretization that adapts to local data geometry while operating entirely on locally randomized signals. To visualize the mechanism's behavior, Figure~\ref{fig:density_reconstruction} displays the function value release under both SLDP and standard LDP. Our contributions are as follows.
\begin{itemize}
\item \textbf{Novel adjacency framework based on privacy regions.} We introduce SLDP, an intermediate privacy framework between central and local DP, where adjacency is defined by allowing a user’s record to move only within a data-dependent $k$-anonymous privacy region.

\item  \textbf{Interactive protocol to discover privacy regions.} We propose an interactive server–user protocol that privately estimates the data-dependent partition using only locally randomized membership signals over a public channel (honest-but-curious server model).

\item \textbf{Privacy guarantees.} We prove that our protocol satisfies $(\varepsilon, \delta)$-DP under the adjacency relation induced by the (implicit) canonical partition.

\item \textbf{Iteration-independent privacy cost.} A key result is that the cumulative privacy loss does not scale with the number of refinement iterations.
    
\item \textbf{Experimental results.} We demonstrate improved utility over standard LDP (and competitive performance vs central DP) on mean estimation, classification, and spatial range-query tasks on synthetic and real-world datasets.
\end{itemize}
\begin{figure}[!t]
    \centering
    \includegraphics[width=0.95\textwidth]{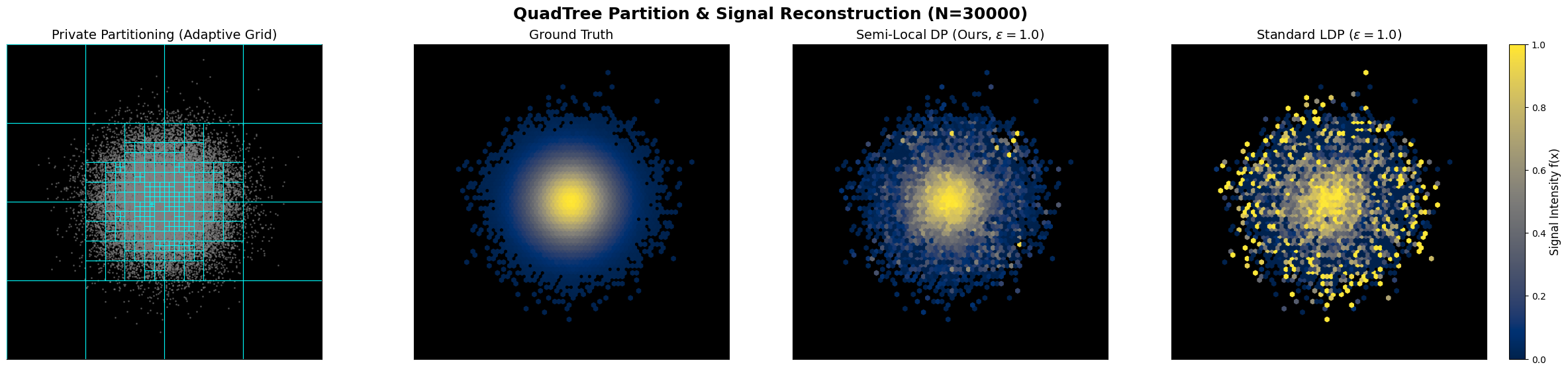}
    \caption{\textbf{SLDP function-value release ($N=3\cdot 10^4$ points).} \textit{From left to right:} the private data-adaptive quadtree partition $\hat{\mathcal{F}}_{\mathcal{P}}(\mathbf{X})$ produced by SLDP (left); ground-truth signal $f(x)=\exp(-\|x\|^2/(2\sigma^2))$ evaluated at the samples; noisy function-value reports under our SLDP with total privacy budget $\varepsilon=1.0$ (and $\delta=0.05$ for SLDP); standard LDP with $\varepsilon=1.0$.}
    \label{fig:density_reconstruction}
\end{figure}

\paragraph{Organization of the paper.} We begin by reviewing the necessary background in Section~\ref{sec:background}. Section~\ref{sec:methodology} details our proposed method, followed by its theoretical analysis in Section~\ref{sec:theory}. Finally, Section~\ref{sec:experiments} presents the experimental evaluation.

\section{Background}
\label{sec:background}
Assume $N$ users with features embedded in some measurable space $\mathcal{X}$, i.e., $X_i \in \mathcal{X}$ describes the $i$-th user. Denote $\X := (X_1,\dots,X_N) \in \mathcal{X}^N$. In the DP setting, a trusted curator collects users' data, aggregates it, and releases the result via a randomized mechanism $\mathscr{M} \colon \mathcal{X}^N \to \mathcal{A}$, where $\mathcal{A}$ is a measurable space of outputs \cite{adnan2022federated, el2022differential, ponomareva2023dp}.  $\mathscr{M}$ satisfies $(\varepsilon,\delta)$-DP guarantee if, for all neighboring datasets $\X$, $\X' \in \mathcal{X}^{N}$ that differ in exactly one individual's data (replacement adjacency), and for all measurable $S \subseteq \mathcal{A}$,
\begin{equation}
\label{def:DP}
\P\lrc{\mathscr{M}(\X) \in S} 
\le e^{\varepsilon} \P\lrc{\mathscr{M}(\X') \in S} + \delta .
\end{equation} 
A canonical example is the DP mean estimator. 
\begin{example}[DP mean estimator]
\label{example:DP}
 Let $f:~\mathcal{X}~\mapsto~[0, C]$ be a bounded function. The goal is to privatly estimate the true mean $\mu := \frac{1}{N}\sum_{i}f(X_i)$. The global sensitivity of the mean is
\[
\Delta_{\mu} = \frac{1}{N} \max_{x, x' \in \mathcal{X}} |f(x) - f(x')| \le \frac{C}{N}.
\]
Let $\xi \sim \mathrm{Lap}\left(0, \frac{C}{N\varepsilon} \right)$. The DP estimator is 
\[
    \hat{\mu}_{\mathrm{DP}}(\X) := \mathscr{M}(\X) = \frac{1}{N}\sum^N_{i=1} f(X_i) + \xi.
\]
\end{example}
While the centralized DP model achieves high statistical utility, its efficiency in spatial tasks often depends on domain partitioning. A notable approach in this direction is PrivTree \cite{zhang2016privtree}, which utilizes a recursive splitting strategy to refine the grid based on noisy data counts. However, PrivTree assumes the presence of a trusted curator.

LDP is a decentralized variant of DP where each user perturbs their data locally before submission. Specifically, a randomized mechanism $\mathscr{M}\colon \mathcal{X} \to \mathcal{A}$ satisfies $\varepsilon$-LDP if 
for all pairs of inputs $x, x' \in \mathcal{X}$ and for any measurable $S \subseteq \mathcal{A}$
\begin{equation}
\label{def:LDP}
\mathbb{P}\left\{\mathscr{M}(x) \in S\right\} \le e^{\varepsilon} \cdot \mathbb{P}\left\{\mathscr{M}(x') \in S\right\}.
\end{equation}
In the LDP setting, each user $i \in [N]$ independently applies an instance of the mechanism, denoted as $\mathscr{M}_i(X_i)$. 
\begin{example}[LDP mean estimator]
\label{example:LDP}
Each user computes $\mathscr{M}_i(X_i) := f(X_i) +\xi_i$, where $\xi_i \sim \mathrm{Lap}(0, \frac{C }{\varepsilon})$ are independent. The curator aggregates the global estimate:
\[
\hat{\mu}_{\mathrm{LDP}}(\X) = \frac{1}{N}\sum^N_{i=1}\mathscr{M}_i(X_i) = \frac{1}{N}\sum^N_{i=1}\left( f(X_i) + \xi_i\right).
\] 
\end{example}

Note that in the LDP setting, the noise variance scales as $O(1/(N\varepsilon^2))$, while in DP it is $O(1/(N\varepsilon)^2)$ (see Example~\ref{example:DP}), leading to a utility gap. To mitigate this gap, the Shuffle Model was introduced \cite{bittau2017prochlo, erlingsson2019amplification}. It introduces an intermediate step where a trusted shuffler randomly permutes the users' reports before they reach the curator. Effectively, the Shuffle Model brings the utility closer to the centralized setting without requiring a trusted curator.

The LDP setting described above typically assumes a non-interactive protocol. However, this model admits rich interaction patterns \cite{joseph2019role}. Interactivity can yield provable advantages over fully non-interactive mechanisms \cite{joseph2019role}. For instance, adaptive hierarchical refinement has been used to improve spatial query accuracy under LDP \cite{du2021ahead, alptekin2025hierarchical}. However, fully interactive protocols must allocate the total privacy budget across rounds, often shrinking the per-round budget and affecting utility \cite{joseph2019role}.

Another line of work improves utility for spatial data by tailoring the privacy definition to the geometry of the domain. 
In particular, geo-indistinguishability (Geo-DP) \cite{andres2013geo}. Let $(\mathcal{X}, d)$ be a metric space. A mechanism $\mathscr{M}\colon \mathcal{X} \to \mathcal{A}$ satisfies $\varepsilon$-Geo-DP if for any pair of inputs $x, x' \in \mathcal{X}$ and any measurable set $S \subset \mathcal{A}$,
\begin{equation}
\label{def:geo_DP}
\P\lrc{\mathscr{M}(x) \in S} 
\le e^{\varepsilon \cdot d(x, x')} \P\lrc{\mathscr{M}(x') \in S}.
\end{equation}
However, while Geo-DP improves utility for metric-based data by relaxing the users' indistinguishability requirement for distant points, it still lacks distributional adaptivity.

\section{Methodology}

\label{sec:methodology}
We propose a density-aware partitioning strategy that adapts to local geometry by selectively refining a fixed, data-independent hierarchy of cells. Specifically, we fix a canonical hierarchical partitioning scheme $\mathcal{P}$ of $\mathcal{X}$ chosen \emph{a priori} (e.g., a depth-$T$ quadtree over $[0,1]^d$), which specifies the admissible splits and the set of candidate cells at all resolutions.

Given a dataset $\X$, we derive a data-dependent leaf partition $\mathcal{F}^*(\mathbf{X}) := \{F^*_1, \dots, F^*_m\}$ by recursively refining cells according to $\mathcal{P}$. To capture local density, we pursue maximal refinement subject to a $k$-anonymity constraint. Specifically, a cell $F_j$ is split if and only if all its children $F^{\mathrm{ch(j)}}_1, \dots F^{\mathrm{ch(j)}}_q$ contain at least $k$ points. Consequently, the refinement stops when a split would result in any child having fewer than $k$ samples
\[
\exists ~F^{\mathrm{ch(j)}}_i:\quad |F^{\mathrm{ch(j)}}_i \cap \mathbf{X}| < k. 
\]
This stopping rule ensures the $k$-anonymity of the resulting partition $\mathcal{F}^*(\X)$,
\[
\forall ~F^*\in\mathcal{F}^*(\mathbf{X}):\quad |F^*\cap \mathbf{X}|\ge k.
\]
By construction, the collection $\mathcal{F}^*(\mathbf{X})$ constitutes a complete partition of $\mathcal{X}$, $\bigsqcup_{j=1}^m F^*_j = \mathcal{X}$. For each individual $i$, we define its privacy region $U^*_i(\mathbf{X})$ as the unique cell $F^*_j \in \mathcal{F}^*(\mathbf{X})$ that contains the corresponding feature $X_i$. 

In standard non-metric DP, which is geometry-agnostic, the privacy region extends to the entire domain $\mathcal{X}$. In Geo-DP, the privacy region is a ball of fixed radius. Our framework employs an adaptive neighborhood whose volume decreases as local density grows (see Fig.~\ref{fig:density_reconstruction}). 

This defines a privacy-region-based adjacency: two datasets are considered neighbors if they differ only by the location of a single user within the same privacy region.
\begin{definition}[Local Neighbors]
\label{def:neighboring_datasets}
Let $\mathcal{F}^*(\mathbf{X})$ be the canonical partition of dataset $\mathbf{X}$. We define the set of neighboring datasets $\mathcal{N}_{F}(\mathbf{X})$ as:
\begin{equation}
\label{eq:neighbors}
\begin{split}
    \mathcal{N}_{F}(\X) := \biggl\{ & (\X, \X') \in \mathcal{X}^N \times \mathcal{X}^N \;\bigg|\; 
    \exists i \in [N]: X_j = X'_j,~\forall j \neq i~\text{and}~ 
        X'_i \in U^*_i(\X) 
    \biggr\}.
\end{split}
\end{equation}
\end{definition}

\subsection{Discovery of Privacy Regions}
\label{sec:discovery}
The canonical privacy regions $U_1^*(\mathbf{X}), \dots, U_N^*(\mathbf{X})$ are induced by the data-dependent partition $\mathcal{F}^*(\mathbf{X})$.
In practice, however, $\mathcal{F}^{*}(\mathbf{X})$ cannot be released as is, since its boundaries directly encode sensitive information about the underlying data density.
To address this, we first release a private estimate $\hat{\mathcal{F}}(\mathbf{X})$ of the partition and subsequently execute the main mechanism conditional on $\hat{\mathcal{F}}(\mathbf{X})$

To construct the private partition $\hat{\mathcal{F}}(\mathbf{X})$, we consider an honest curator operating over an untrusted communication channel. Specifically, we assume an honest-but-curious server that follows the protocol and observes the entire execution transcript. We consider an open communication channel: all messages (user reports and server broadcasts) are available for an external observer. 

The curator coordinates the recursive refinement while ensuring that each cell maintains $k$-anonymity based on noisy observations. The process begins with a coarse initial grid over $\mathcal{X}$, which is iteratively refined according to a predefined hierarchical partitioning scheme $\mathcal{P}$. In each step, the curator broadcasts a refinement query for all cells. Users respond with locally-noised membership indicators. By aggregating these noisy responses, the curator estimates the count of samples in each candidate child cell. The refinement proceeds as long as every child cell is estimated to contain at least $k$ samples.

This communication protocol offers a compelling middle ground within the privacy spectrum. While it satisfies the formal requirements of DP (see Theorem~\ref{thm:DP}), the use of local randomization by users ensures a more robust trust model than standard Central DP, as it removes the need for a fully trusted curator. However, because the protocol is interactive and requires central coordination, it does not meet the stricter conditions of Local Differential Privacy (LDP). This ``intermediate'' position allows us to introduce a formal privacy framework, which we define as Semi-Local Differential Privacy (SLDP).

\begin{definition}[Semi-Local Differential Privacy (SLDP)]
\label{def:SLDP}
    Let $\mathscr{M}$ be an interactive protocol between a server and $N$ users. Let $\X \in \mathcal{X}^N$ be the users' private data. We denote as $O(\X)$ the \emph{random transcript} of all messages received by the server during the execution of the communication protocol $\mathscr{M}$ on $\X$.
    
    We say that $\mathscr{M}$ satisfies $(\varepsilon, \delta)$-SLDP if for any dataset $\mathbf{X}$ and any neighbor $\mathbf{X}'$ such that $(\mathbf{X}, \mathbf{X}') \in \mathcal{N}_{F}(\mathbf{X})$, for all measurable sets of transcripts $S \subseteq \mathcal{A}$ holds 
    \begin{equation}
        \label{eq:sldp_definition}
        \mathbb{P}\left\{O(\X) \in S\right\} \le e^\varepsilon \cdot \mathbb{P}\left\{O(\X') \in S\right\} + \delta.
    \end{equation}

\end{definition}
\begin{example}[SLDP mean estimator]
\label{example:SLDP}
Each user knows their privacy region $\hat{U}_i(\X)$ and computes $\mathscr{M}_i(X_i) := f(X_i) +\xi_i$, where $\xi_i \sim \mathrm{Lap}(0, \frac{\Delta_{i, f}}{\varepsilon})$ and 
\[
\Delta_{i, f} := \max_{x, x' \in \hat{U}(\X_i)} |f(x) - f(x')|.
\]
The curator aggregates the global estimate:
\[
\hat{\mu}_{\mathrm{SLDP}}(\X) = \frac{1}{N}\sum^N_{i=1}\mathscr{M}_i(X_i) = \frac{1}{N}\sum^N_{i=1}\left( f(X_i) + \xi_i\right).
\] 
\end{example}

We conclude this section by detailing the two-party protocol for constructing $\hat{\mathcal{F}}(\mathbf{X})$. The procedures for the client and the server are outlined in Algorithm~\ref{alg:client-side} and Algorithm~\ref{alg:server-side}, respectively.

\begin{algorithm}[tb]
\caption{Client-Side Protocol}
\label{alg:client-side}
\begin{algorithmic}
    \STATE {\bfseries Input:} User feature $X_i\in \mathcal{X}$
    \STATE $t \gets 0$, $U_t \gets \mathcal{X}$ \qquad \COMMENT{Current privacy region for the user}
    \STATE Receive $\varepsilon$ from the server
    \LOOP
        \STATE $t\gets t+1$
        \STATE Recieve $\mathcal{F}_t = \{F_{t}^{j}\}_{j\in [J_t]}$ \qquad  \COMMENT{$\mathcal{F}_t$ is a current frontier}
        \IF{$X_i \not \in \sqcup_{j} F^{j}_t$}
            \STATE \textbf{Return} $U_{t-2}$ \qquad  \COMMENT{If $t=1$, $\sqcup_{j} F^{j}_t = \mathcal{X}$ and $X_i \not \in \sqcup_{j} F^{j}_t$ holds}
        \ENDIF
        \FOR{all $j$: $F^{j}_t \subset U_{t-1}$}

            \STATE $I^{j}_{i} \gets \mathbf{I}\{X_i\in F^{j}_t\}$ \qquad  \COMMENT{Indicator for the $j$-th frontier cell}
            \STATE $Y^{j}_{i} \gets I^{j}_{i} + \xi^j_{i} $, where $\xi_i \sim \mathrm{Lap}(0, \frac{1}{\varepsilon})$
            \IF{$I^{j}_{i} = 1$}
                \STATE $U_t \gets F^j_t$
            \ENDIF
        \ENDFOR
        \STATE Send $Y_i = \{Y^{j}_{t, i}\}_{j \in [J_t]}$ to Server
    \ENDLOOP
\end{algorithmic}
\end{algorithm}

\begin{algorithm}[tb]
\caption{Server-Side Protocol}
\label{alg:server-side}
\begin{algorithmic}
    \STATE {\bfseries Input:} Partitioning scheme $\mathcal{P}$; stopping parameters $Q,\,T$; budget $\varepsilon$ and $x(\delta)$

    \STATE {\bfseries Output:} $\hat{\mathcal{F}}(\X)$
    
    \STATE {\bfseries Initialization}
    \STATE Round number $t \gets 0$; indicator of active cells $a_t \gets (1)$; partition of active cells $\mathcal{F}^{\mathrm{active}}_t \gets \{\mathcal{X}\}$
    \STATE Broadcast $\varepsilon$ to all users
    
    \STATE {\bfseries Adaptive Split}
    \WHILE{$\mathcal{F}_{t}^{\mathrm{active}} \neq \varnothing$ and $t\le T$}
        \STATE $t \gets t + 1$
        \STATE Construct a partition of $\mathcal{F}_{t-1}^\mathrm{active}$:\, $\mathcal{F}_{t} \gets \{F^{j}_t\}_{j \in [J_t]}$ such that $\mathcal{F}_{t}^{\mathrm{active}} = \sqcup_{j \in [J_t]} F^{j}_t$
        \STATE Broadcast $\mathcal{F}_{t}$ to all users
        \STATE Receive responses from active users $\{Y^{j}_{t, i}\}_{i,j}$ from users

        \FOR{$F^{k}_{t-1} \in \mathcal{F}_{t1}^{\mathrm{active}}$} 
        \STATE $n(F^{k}_{t-1}) \gets \# \{i : Y^{j}_{t, i} \text{ received for all } F_t^j \subset F_{t-1}^k\}$ \qquad  \COMMENT{Estimate active users in cells of $\mathcal{F}_{t-1}^{\mathrm{active}}$} 
        \ENDFOR
        
        \STATE $a_{t-1} \gets \underbrace{(1, 1, 1 ...)}_{J_{t-1}~\text{times}}$ 
        \FOR{$F^{j}_t \in \mathcal{F}_t$} \STATE \COMMENT{This for-loop detects active cells in $\mathcal{F}_{t}$ }\,
            \STATE $\hat{n}(F^{j}_t) \gets \sum_{i} Y^{j}_{t, i}$
            \STATE $\Delta^{j}_t \gets \dfrac{1}{\varepsilon} \left(\sqrt{2n(F^{p(j)}_{t-1}) x(\delta)} + 6x(\delta)\right)$ 
         \STATE   \COMMENT{$F^{p(j)}_{t-1}$ is a ``parental'' cell of $F^j_t$, $p(j)\in[J_{t-1}]$}
            \IF{$\hat{n}(F^{j}_t) < Q + \Delta_t$}
                \STATE $a_{t-1}[p(j)] \gets 0$
                \STATE add $F^{p(j)}_{t-1}$ to $\hat{\mathcal{F}}(\X)$
            \ENDIF
        \ENDFOR
        \STATE $\mathcal{F}_t^{\mathrm{active}} \gets \{F_t^j : a_{t-1}[p(j)] = 1\}$
    \ENDWHILE
\end{algorithmic}
\end{algorithm}

\section{Theoretical analysis}
\label{sec:theory}
In this section, we provide the theoretical justification for the adaptive refinement procedure described in Algorithms~\ref{alg:client-side} and ~\ref{alg:server-side}. Specifically, we derive a high-probability concentration bound on the estimation error of the cell counts. This analysis is crucial for determining the correction term $\Delta_t$ (Line 14 of Algorithm~\ref{alg:server-side}), ensuring that the splitting decision is robust to the injected Laplacian noise.
\begin{lemma}
    Fix an arbitrary iteration $t \in [T]$ and let the corresponding frontier be $\mathcal{F}_t = \sqcup_{j\in [J_t]} F^{j}_t$. Fix some cell $F^{j}_t$ with its parent cell $F^{p(j)}_{t-1}$. Let $I_{p(j)}$ be the set of active users in the parent cell. Let $n(F^j_t)$ be the true number of users within cell $F^j_t$ and let $\hat{n}(F^j_t)$ be an estimator (see line 13 in Algorithm~\ref{alg:server-side}). For any $x > 0$ it holds with probability at least $1-e^{-x}$, that
    \[
    \left|n(F^j_t)  - \hat{n}(F^{j}_t) \right| \le \Delta^{j}_t := \frac{1}{\varepsilon}\sqrt{2 n(F^{p(j)}_{t-1})} + \frac{6x}{\varepsilon},
    \] 
    where $n(F^{p(j)}_{t-1}) := |I_{p(j)}|$.
\end{lemma}
\begin{proof}
By construction, $n(F^j_t) =  \sum_{i \in I_{p(j)}} \mathbf{I}\left\{X_i \in F^j_t \right\}$. The estimated number of points $\hat{n}(F^j_t)$ inside $F^j_t$ (see line 15 in Algorithm~\ref{alg:server-side}) is written as 
\begin{align*}
\hat{n}(F^j_t) := \sum_{i \in I_{p(j)}} Y^j_{t, i} = \sum_{i \in I_{p(j)} } \left(\mathbf{I}\left\{X_i \in F^j_t \right\} + \xi^j_i\right),
\end{align*}
where $\xi^j_i \sim \mathrm{Lap}\left(0 , \, \tfrac{1}{\varepsilon}\right)$. To account for the deviation of the estimator $\hat{n}(F^j_t)$, we introduce $\Delta^j_t$. Specifically, we apply
Corollary 3.1 from \cite{kroshnin2024bernstein} for $\sum_{i\in I_{p(j)}} \xi^j_i$. Note that $\|\xi^j_i\|_{\psi_1} =\frac{2}{\varepsilon}$, $\mathrm{Var}(\xi^j_i) = \frac{2}{\varepsilon^2}$, $\mathbb{E}\,\xi^j_i = 0$. So, we get with probability at least $1-e^{-x}$
\[
\sum_{i\in I_{p(j)}}\xi^j_i \le \frac{1}{\varepsilon}\sqrt{2 n(F^{p(j)}_{t-1})} + \frac{6x}{\varepsilon}.
\]
Note that for any $F^t_j$ the proposed procedure reveals $n(F^{p(j)}_{t-1})$ at the step $t$ (see line 12 in Algorithm~\ref{alg:server-side}). 
\end{proof}

Having analyzed the estimation error, we now turn to the formal privacy guarantee. The following theorem establishes that the interactive protocol satisfies the definition of SLDP.
\begin{theorem}[Algorithm is SLDP]
\label{thm:DP}
Let $\mathcal{F}^*(\X)$ be a true unknown data partition w.r.t. some $\mathcal{P}$. Let $\mathscr{M}$ be the randomized mechanism from Algorithms~\ref{alg:client-side} and~\ref{alg:server-side} that estimates $\hat{F}(\X)$. 
Then $\mathscr{M}$ satisfies DP w.r.t. $\mathcal{N}_{F}(\X)$. That is, let $\mathcal{A}$ be the space of all admissible transcripts. For any $\X'$ s.t. $ (\X, \X') \in \mathcal{N}_{F}(\X)$ and measureable set $S \subseteq \mathcal{A}$ it holds
\[
\mathbb{P}(\mathscr{M}(\X
)\in S) \le e^\varepsilon\, \mathbb{P}(\mathscr{M}(\X'
)\in S) + \delta.
\]
\end{theorem}
\begin{proof}
    The protocol is executed in rounds. Let $t_i \in \mathbb{N}$ be the number of rounds when the user $i$ is active. Let the full transcript of users' messages to the server be $o$,
\[
o = \left( (y^{j}_{k,i})_{j \in J_k} \right)_{i \in [N],\, 1 \le k \le t_i},
\]
where $J_k$ indexes the set of frontier cells queried in round $k$, and $y^{j}_{k,i}\in\mathcal Y$ being the responce of $i$-th user at $k$-th round regarding the frontier cell $j \in J_{k}$.

From now on, we fix some user index $i_0 \in [N]$. Note that
$\mathcal{A}$ admits the decomposition $\mathcal{A} := \sqcup_{t} \mathcal{A}_t$, where $\mathcal{A}_t := \{o \in \mathcal{A}\,|\, t_{i_0} = t \}$, i.e., $\mathcal{A}_t$ is the set of all possible communication transcripts where the user $i_0$ is active up to round $t$. 
Next, we introduce a transcript decomposition by time: for a round index $\tau \ge 0$ and a transcript $o$, we denote
\begin{gather*}
o_{\le \tau} = \left( (y^{j}_{k,i})_{j \in J_k} \right)_{i \in [N],\, 1 \le k \le \min(t_i, \tau)}, \quad
o_{> \tau} = \left( (y^{j}_{k,i})_{j \in J_k} \right)_{i \in [N]: t_i > \tau,\, k > \tau}.
\end{gather*}
In other words, $o_{\le \tau}$ is the transcipt up to moment $\tau$, and $o_{> \tau}$ is the transcript after $\tau$.
So, we write $o = (o_{\le \tau}, o_{> \tau})$.

Further, let the observed (random) communication log be 
\[
\mathscr{M}(\X) := O = \left( (Y^{j}_{k,i})_{j \in J_k} \right)_{i \in [N],\, 1 \le k \le T_i},
\]
where $T_i$ is the observed time of communication interruption for any user $i$. 

Consider a neighboring dataset $\X'$ s.t. some $X'_{i_0} \in U^*_{i_0}(\X)$ replaces $X_{i_0}$. Let the corresponding observed communication transcript be $\mathscr{M}(\X') := O'$. Denote as $\tau_{i_0}$ the ideal stopping time of the communication of the user $i_0$, i.e.\ such that $U^*_{i_0}(\X) = F_{\tau_{i_0}}^j$ for some $j$. Since $X'_{i_0} \in U^*_{i_0}(\X)$, the sequence of indicator values generated by the user in Algorithm~\ref{alg:client-side} is identical for both $X_{i_0}$ and $X'_{i_0}$ for any sequence of server queries leading to the leaf covering $U^*_{i_0}(\mathbf{X})$ before the stopping time $\tau_{i_0}$. Thus, we get for any $t\le \tau_{i_0}$
\begin{equation}
\label{eq:communication_seimilarity}
O_{\le t} \overset{d}{=} O'_{\le t}, \quad 
\mathbb{P} \left\{O \in S  \cap \mathcal{A}_t \right\} = \mathbb{P} \left\{O' \in S  \cap \mathcal{A}_t \right\} ~.
\end{equation}
Now, we consider the case $t = \tau_{i_0} + 1$. We set 
\[
p\left(O_{\le t} \right) := \mathbb{P} \left\{(O_{\le t},  O_{> t}) \in S  \cap \mathcal{A}_t \middle| O_{\le t} \right\}. 
\]
Note, that
\begin{align*}
O_{\le t} = \left(O_{< t}, (Y^j_{t, i})_{i: t_i \ge t, j \in J_t}\right) = \left(O_{< t}, (Y^j_{t, i})_{\black{i \in [N]~\text{s.t.} i\neq i_{0}}, j \in J_t}, (Y^j_{t, i_0})_{j \in J_t}\right). 
\end{align*}
And let's denote for simplicity
\[
\tilde{O}_{t, i\neq i_0} := (Y^j_{t, i})_{\black{i \in [N]~\text{s.t.} i\neq i_{0}}, j \in J_t} , \quad \tilde{O}_{t, i_0} := (Y^j_{t, i_0})_{j \in J_t}.
\]
Consequently, $p(O_{\le t})$ can be written as
\[
p\left(O_{\le t} \right) = p\left(O_{< t}, \tilde{O}_{t, i\neq i_0}, 
\tilde{O}_{t, i_0} \right) \ge 0.
\]
By analogy with the observed user responses at round $t$ in dataset $\X$, we define a corresponding term for the neighboring dataset $\X'$ as follows
\[
\tilde{O}'_{t, i\neq i_0} := ({Y'}^j_{t, i})_{\black{i \in [N]~\text{s.t.} i\neq i_{0}}, j \in J_t} , \quad \tilde{O}'_{t, i_0} := ({Y'}^j_{t, i_0})_{j \in J_t}.
\]
Note that by construction of the procedure, it holds
\begin{equation}
    \label{eq:nieghb_communication_at_t}
    \tilde{O}_{t, i\neq i_0} \overset{d}{=} \tilde{O}'_{t, i\neq i_0}.
\end{equation}
Now we consider
\begin{align*}
   &\mathbb{P} \left\{O \in S  \cap \mathcal{A}_t \right\}  = \mathbb{E}\, p\left(O_{\le t} \right)
= \mathbb{E}\,\mathbb{E}\,
   \left[
    p\left(O_{< t}, (\tilde{O}_{t, i\neq i_0}, \tilde{O}_{t, i_0}) \right) \middle| O_{< t}, \tilde{O}_{t, i\neq i_0}
   \right] \\
   & \overset{(a)}{\le} \mathbb{E}\, \black{e^\varepsilon} \, \mathbb{E}\, \left[p\left(O_{< t}, (\tilde{O}_{t, i\neq i_0}, \tilde{O}'_{t, i_0}) \right) \middle | O_{< t}, \tilde{O}_{t, i\neq i_0} \right ] \overset{(b)}{\le} \mathbb{E}\, \black{e^\varepsilon} \, \mathbb{E}\, \left[p\left(O'_{\le t} \right) \middle | O'_{< t}, \tilde{O}'_{t, i\neq i_0} \right ]=\black{e^\varepsilon} \,\mathbb{P} \left\{O' \in A_t \right\};
\end{align*}
(a) holds because we fix $O_{< t}$ and $\tilde{O}_{t, i\neq i_0}$ and use \black{the fact that each round of the procedure is DP by construction,} so $\mathbb{E}\, f\left(\tilde{O}_{t, i_0}\right) \le \black{e^\varepsilon} \, \mathbb{E}\, f\left(\tilde{O}'_{t, i_0}\right)$; (b) holds because $t = \tau_{i_0} + 1$ and, consequently, $O_{< t} \overset{d}{=} O'_{< t}$, see \eqref{eq:communication_seimilarity} and \eqref{eq:nieghb_communication_at_t}.

Finally, we consider $ t > \tau_{i_0} + 1$. By construction of $O$, it holds
\begin{align*}
\sum_{t>\tau_{i_0} +1 } \P\left\{ O \in S  \cap \mathcal{A}_t \right\}  =  \mathbb{P}\, \left\{ O \in \sqcup_{t>\tau_{i_0} +1} S  \cap \mathcal{A}_t \right\} = \mathbb{P}\, \left\{ T_{i_0} > \tau_{i_0} + 1 \right\} \le \delta.
\end{align*}
The same holds for $O'$. Taking into account all the aforementioned, we get for any $S \subseteq \mathcal{A}$  
\begin{align*}
    & \P \left\{ O \in S \right\}  = \sum_{t} \left\{ O \in S  \cap \mathcal{A}_t\right\} = \sum^{\black{\tau_{i_0}}
    }_{t = 1} \P\left\{ O \in S  \cap \mathcal{A}_t\right\} + \P\left\{ O \in  S  \cap \mathcal{A}_{\tau_{i_0}+1}\right\}  + \sum_{t>\tau_{i_0} +1 } \P\left\{ O \in  S  \cap A_t \right\} \\
     &\le \sum^{\tau_{i_0}
    }_{t = 1} \P \left\{ O' \in  S  \cap \mathcal{A}_t\right\} + e^{\varepsilon} \P\left\{ O' \in  S  \cap \mathcal{A}_{\tau_{i_0}+1}\right\}  + \delta 
     \le e^{\varepsilon} \P\left\{ O' \in S \right\} + \delta . 
\end{align*}
\end{proof}

\begin{remark}
\label{remark:budget}
For any two neighboring datasets $\X, \X'$, s.t. $\X' \in \mathcal{N}_{F}(\X)$ and any refinement query at depth $t < \mathrm{depth}(U^*_i(\mathbf{X}))$, the user responses follow identical distributions (see eq.~\eqref{eq:communication_seimilarity}). Thus, the privacy loss accumulates only at the final stage of the partitioning. So, our framework allows for constructing partitions of arbitrary depth without effectively consuming the privacy budget during the intermediate steps.
\end{remark}

\section{Experiments}
\label{sec:experiments}
To demonstrate the performance of the algorithm, we consider several examples: mean estimation of a function, classification, and answering spatial queries. To facilitate reproducibility, we provide a self-contained code implementation using synthetic data (see \url{https://github.com/asuvor/privacy}). This allows for immediate execution without the overhead of the complex preprocessing pipeline required for real-world datasets.

\subsection{Mean estimation of a function}
We evaluate the performance of our SLDP against central DP and LDP baselines on a synthetic dataset. The task is to estimate the population mean of the squared norm $f(x) = \|x\|^2$ over a 2D domain.

The construction of corresponding mean estimators are as in Example~\ref{example:DP} (central DP), Example~\ref{example:LDP} (LDP) and Example~\ref{example:SLDP} (SLDP). To construct an SLDP estimator, we split the budget $\varepsilon$: the construction of the privacy regions costs $\tfrac{1}{2}\varepsilon$, and the mean estimation costs $\tfrac{1}{2} \varepsilon$. Data points are sampled from a normal distribution $\mathcal{N}(0, 1.5^2)$ truncated to the bounding box $[-10, 10]^2$ via rejection sampling. We vary the sample size $N \in \{2, 16, 32, 64\}\cdot$ and the total privacy budget $\varepsilon \in \{0.5, 1.0, 2.0, 4.0, 8.0, 12.0\}$. To ensure statistical robustness, we perform $60$ independent simulations for each configuration $(N, \varepsilon)$ and report the Mean Squared Error (MSE) with corresponding confidence intervals. In all experiments $\delta = 0.05$. Fig.~\ref{fig:mse_and_boxplots} demonstrates the results.

\begin{figure*}[!t]
    \centering
    \begin{subfigure}{\textwidth}
        \centering
        \includegraphics[width=\textwidth]{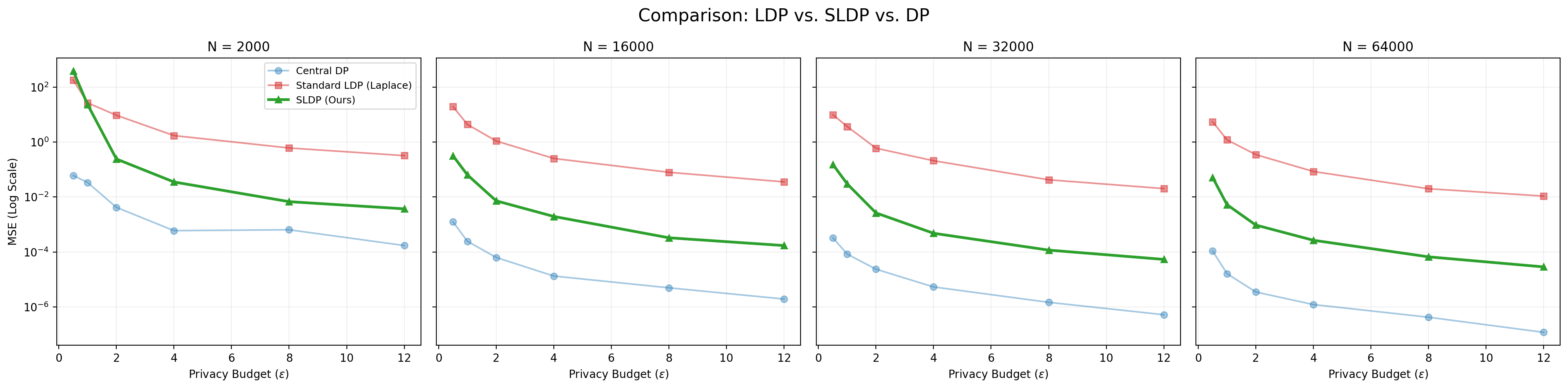}
        \caption{Mean Squared Error (MSE) comparison on synthetic 2D data across varying sample sizes $N$) and privacy budgets $\varepsilon$. SLDP (green) outperforms Standard LDP (pink).}
        \label{fig:mse_results_a}
    \end{subfigure}

    \vspace{0.6em}

    \begin{subfigure}{\textwidth}
        \centering
        \includegraphics[width=\textwidth]{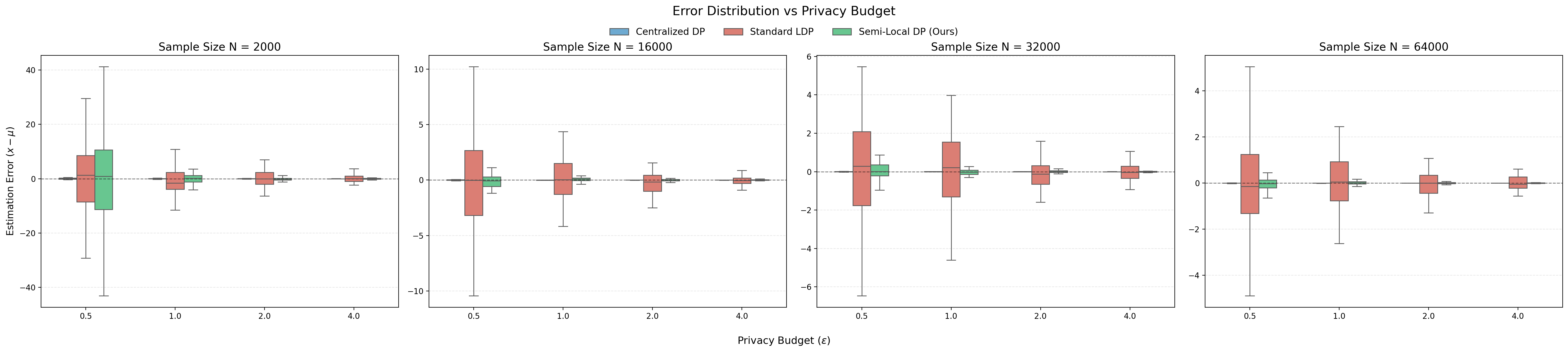}
        \caption{Distribution of estimation errors for SLDP compared to Standard LDP and Central DP across different budgets.}
        \label{fig:error_boxplots_b}
    \end{subfigure}

    \caption{\textbf{Mean estimation performance under privacy.}
    (a) Convergence comparison for Centralized DP, standard LDP, and SLDP.
    (b) Error distributions across $N$ and $\varepsilon$.}
    \label{fig:mse_and_boxplots}
\end{figure*}

\subsection{Classification on California Housing dataset}
\label{sec:exp_clust}
In this section, we empirically evaluate the privacy-utility trade-off of our SLDP against established baselines. 

We utilize the California Housing dataset~\cite{pace1997sparse} as provided by \cite{pedregosa2011scikit}. For the input features, we retain only the geospatial coordinates (latitude and longitude), scaled to the unit square $[0, 1]^2$, such that $X_i \in [0, 1]^2$.
The continuous target variable (median house value) is binarized to define a classification task. Specifically, we assign labels based on the global median value of the dataset: we set $y_i=1$ if the local house value exceeds the global median, and $y_i=0$ otherwise (see Fig.~\ref{fig:housing_data}). We compare four distinct mechanisms:
\begin{itemize}
    \item SLDP-Quantization (Ours): A quadtree-based mechanism where data points are mapped to the centroids of privacy-preserving reiongs. 
    \item SLDP-Split (Ours): A variant where the privacy budget $\varepsilon$ is split evenly ($50/50$). Half of the budget is allocated to the partition construction. The remaining half is used to generate synthetic points: each real point is replaced by a sample drawn from a Laplace distribution centered at the geometric centroid of its privacy region, with the noise scale proportional to its size.
    \item Standard LDP: Independent Laplace noise ($1/\varepsilon$) is added to each coordinate, followed by clipping to the domain boundaries.
    \item Geo-Indistinguishability: The Planar Laplace Mechanism \cite{andres2013geo}. Noise is generated in polar coordinates to satisfy $\varepsilon$-geo-indistinguishability.
\end{itemize}
We use Random Forest, kNN, and Logistic Regression. The data is split into training ($70\%$) and testing ($30\%$) sets. The privacy mechanisms are applied only to the training features. The testing set remains unperturbed. Fig. \ref{fig:classifiers} reports the  F1-Score averaged over 60 independent runs for $\varepsilon \in [0.1, 10]$. 

\paragraph{Hyperparameter Sensitivity.} The structural parameter $k$ governs our SLDP mechanism. We analyze analyze the classificatin performance depending on the choice of $k$; see Appendix~\ref{appx:classification}.
\begin{figure*}[!t]
    \centering
\begin{subfigure}[b]{0.3\linewidth}
    \includegraphics[width = 0.95\linewidth]{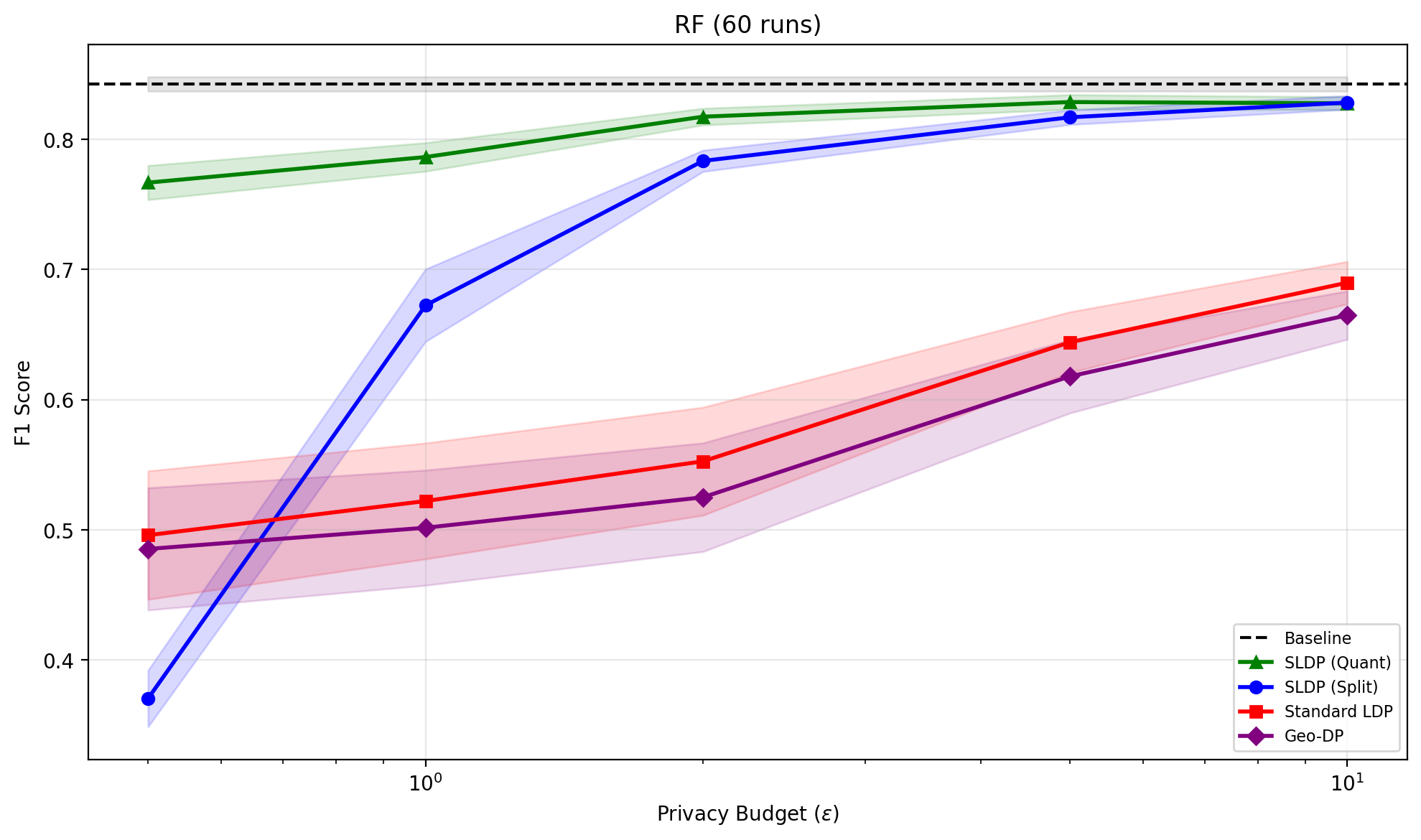}
\end{subfigure}
\begin{subfigure}[b]{0.3\linewidth}
    \includegraphics[width = 0.95\linewidth]{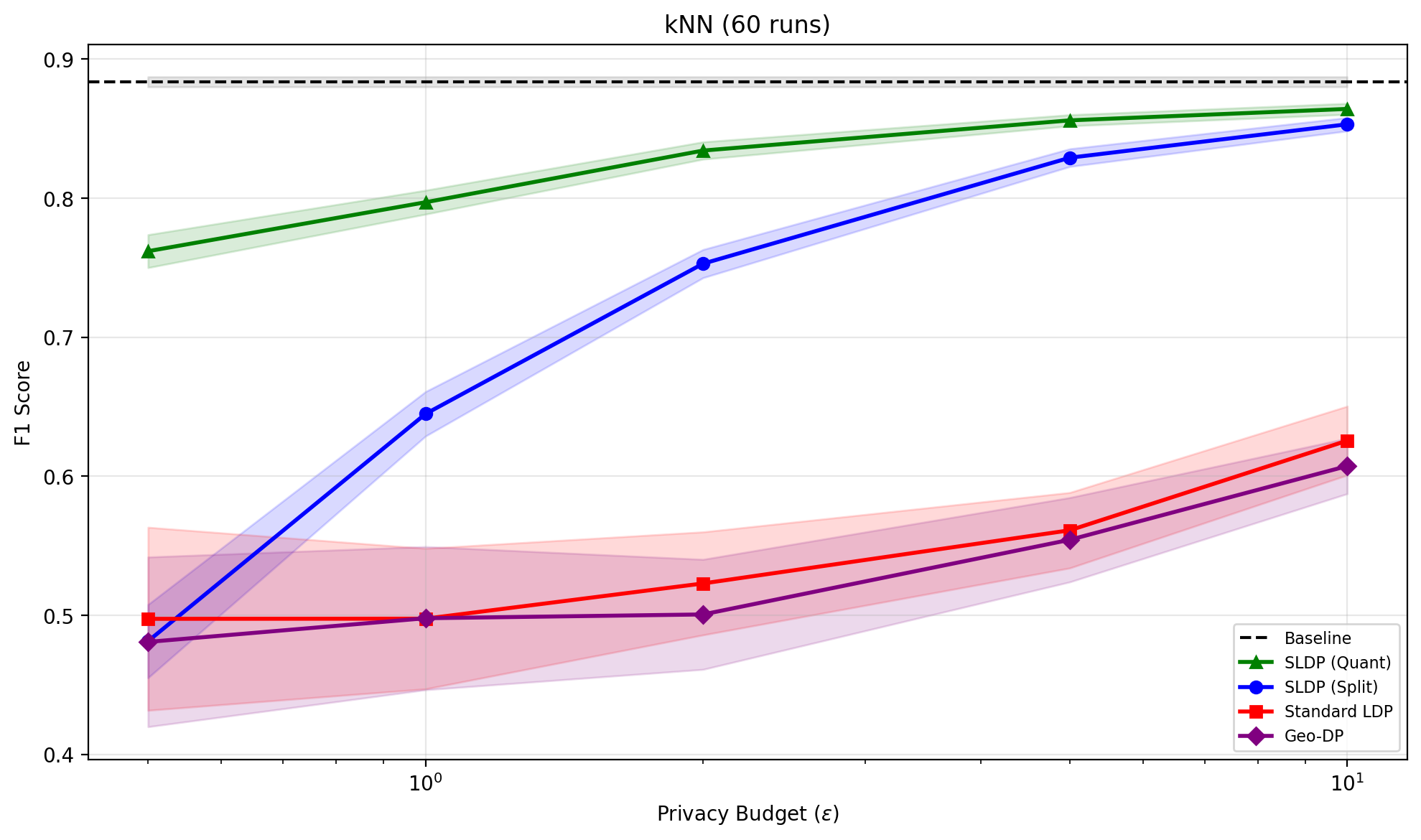}
\end{subfigure}
\begin{subfigure}[b]{0.3\linewidth}
    \includegraphics[width = 0.95\linewidth]{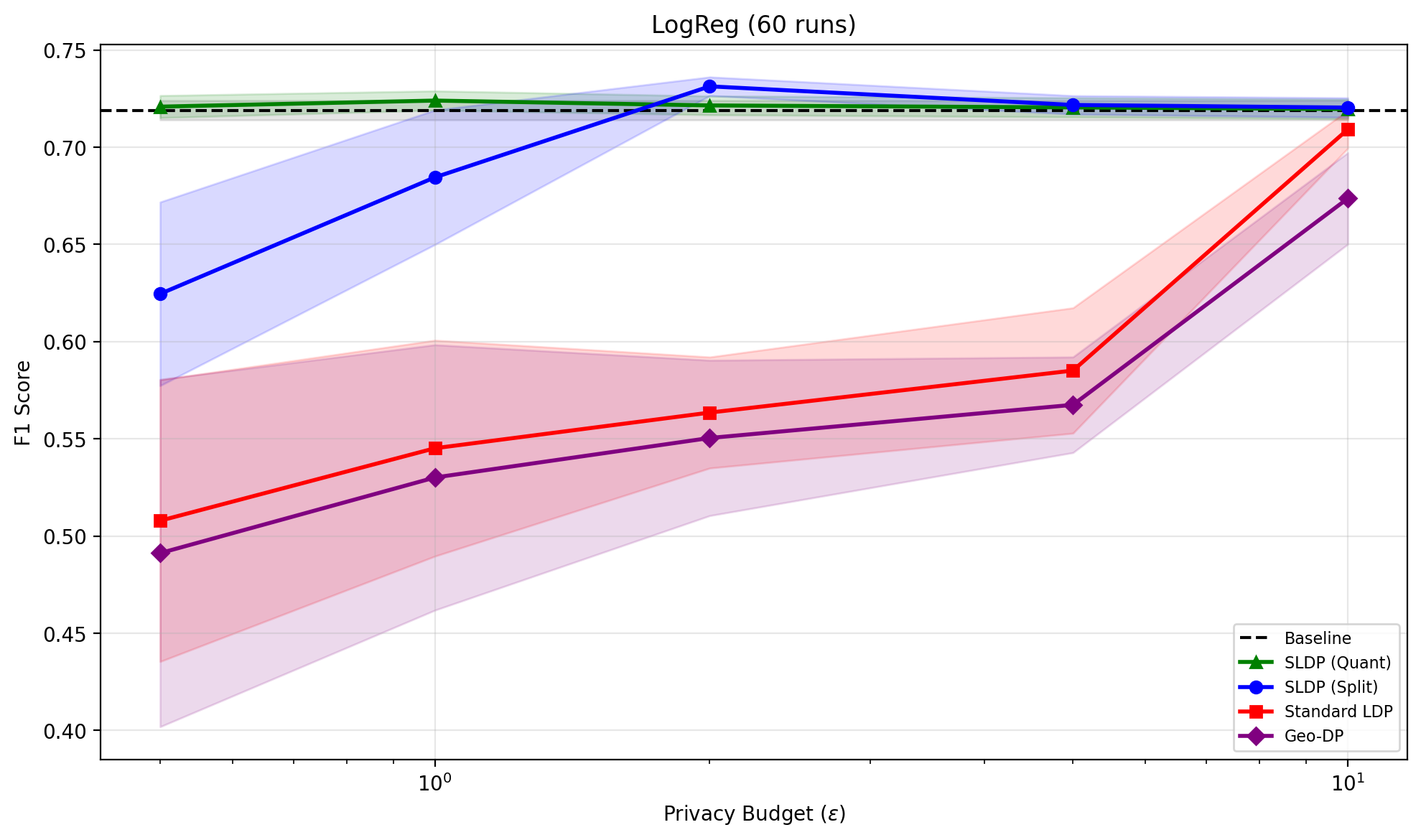}
\end{subfigure}
\caption{ Classification utility (F1-Score) on the California Housing dataset. Comparison of predictive performance across Random Forest, k-NN, and Logistic Regression classifiers trained on data perturbed by SLDP (Ours), Standard LDP, and Geo-Indistinguishability mechanisms. Results are averaged over 60 independent runs with varying privacy budgets $\varepsilon \in [0.1, 10]$. Shaded areas indicate the standard deviation.}
\label{fig:classifiers}
\end{figure*}

\subsection{Answering spatial queries}
We consider $N$ user locations in a normalized 2D domain $[0,1]^2$. We focus on answering a set of rectangular range queries $\mathcal{Q}$. A single query $q \in \mathcal{Q}$ is defined by a rectangular region $R_q = [x_{\min}, x_{\max}] \times [y_{\min}, y_{\max}]$. The true answer to the query, denoted as $y_q$, is the count of users located within the region: $y_q = |\{p \in D \mid p \in R_q\}|.$ We generate a workload of $M=200$ spatial range queries using an anchored multi-scale strategy (see Appendix~\ref{appx:spatial_questions}). We evaluate our method on four real-world geospatial datasets covering diverse mobility patterns. 
Gowalla and Brightkite \cite{cho2011friendship} represent sparse location-based social network (LBSN) check-ins, while Geolife \cite{zheng2009mining} and Porto Taxi \cite{moreira2015taxi} consist of GPS trajectories collected from human and vehicular traffic.

We compare our SLDP against two reference methods: PrivTree (Central DP)~\cite{zhang2016privtree} and LDP-KDTree (Local DP)~\cite{alptekin2025hierarchical}. The choice of hyperparameters for PrivTree is as recommended by ~\cite{zhang2016privtree} (see \eqref{eq:param_priv_tree})

All algorithms produce a set of disjoint cells $\mathcal{F} = \{F_1, \dots, F_k\}$, where each cell $F_i$ isa rectangular region. It holds a noisy count $\hat{c}_i$ of the users within it. The estimated count $\hat{y}_Q$ for a query $Q$ is computed as the weighted sum of counts from all intersecting leaves:
\[
\hat{y}_Q = \sum_{F_i \in \mathcal{L}} \hat{c}_i \cdot \frac{\text{Area}(F_i \cap Q)}{\text{Area}(F_i)}
\]
To assess the performance of each method, we compute private count estimates $\hat{y}$, we compute the Mean Relative Error (MRE) with a smoothing threshold $\tau$:
\[
\text{MRE} = \frac{1}{|Q|} \sum_{q \in Q} \frac{|\hat{y}_q - y_q|}{\max(y_q, \tau)},\]
where $\tau = 10\max(1, 10^{-4}N)$ is used to mitigate the impact of small counts on the error metric. For statistical robustness, all results are reported as averages over 60 independent trials. The experimental procedure for each trial is as follows: (i) a random subset of $N$ users (e.g., $N \in \{5k, 20k\}$) is sampled from the full dataset; (ii) a new set of anchored queries is generated based on the sampled subset; (iii) the privacy algorithms are executed with varying privacy budgets; (iv) errors are computed for the batch of queries and averaged across the 60 repetitions. Fig.~\ref{fig:spatial_queries_error} provides the results.
\begin{figure}[!t]
    \includegraphics[width=\linewidth]{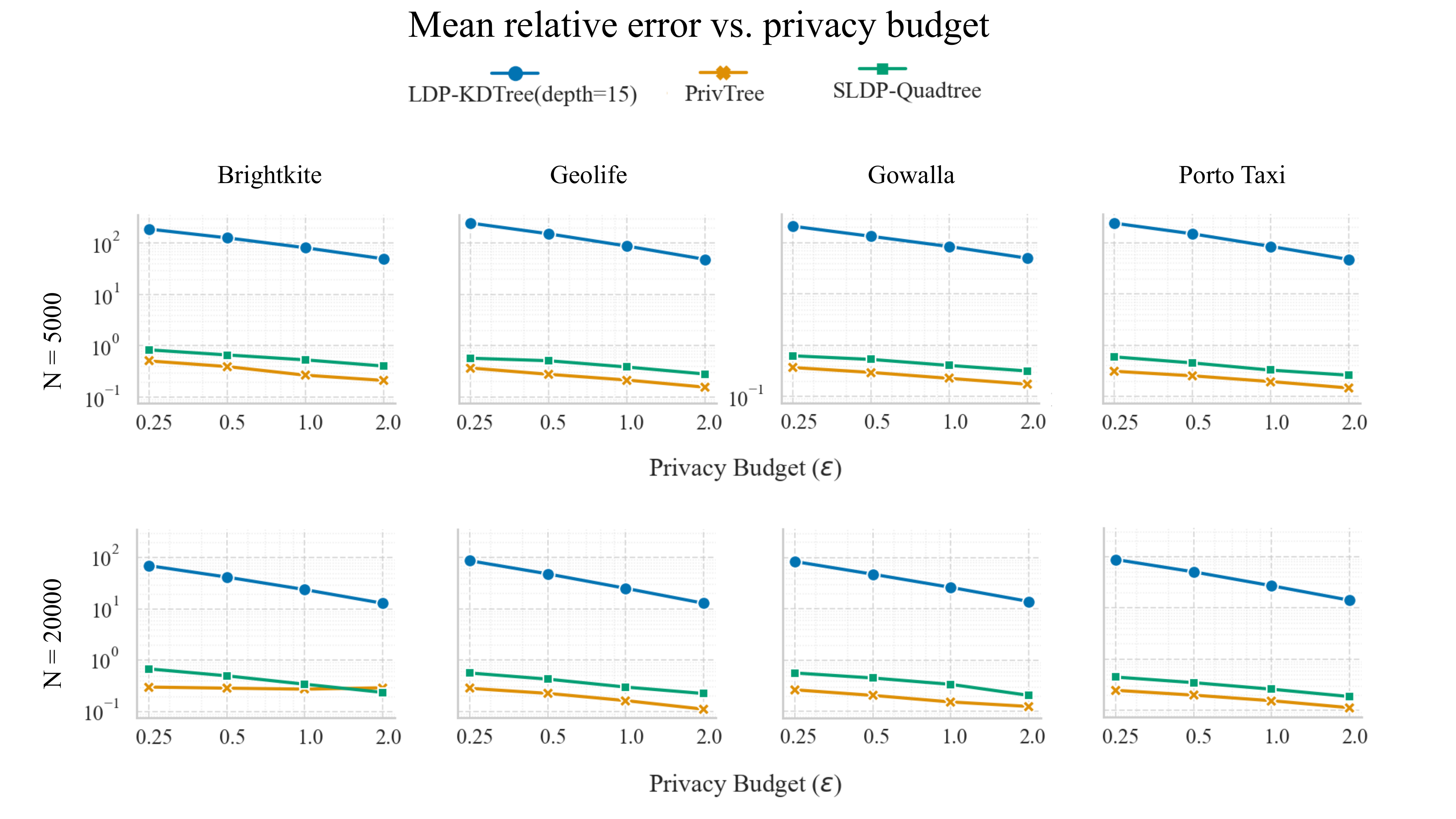}
    \caption{\textbf{Accuracy of spatial range queries on real-world datasets.} Mean Relative Error (MRE) vs. privacy budget ($\varepsilon$) for Brightkite, Geolife, Gowalla, and Porto Taxi datasets for ($N\in\{5\cdot 10^3, 20\cdot 10^3\}$).}
    \label{fig:spatial_queries_error}
\end{figure}

\bibliographystyle{unsrt}  
\bibliography{references}

\appendix
\newpage
\section{Experiments}

\subsection{Classification}
\label{appx:classification}
\begin{figure}[!ht]
    \centering
    \includegraphics[width=0.5\linewidth]{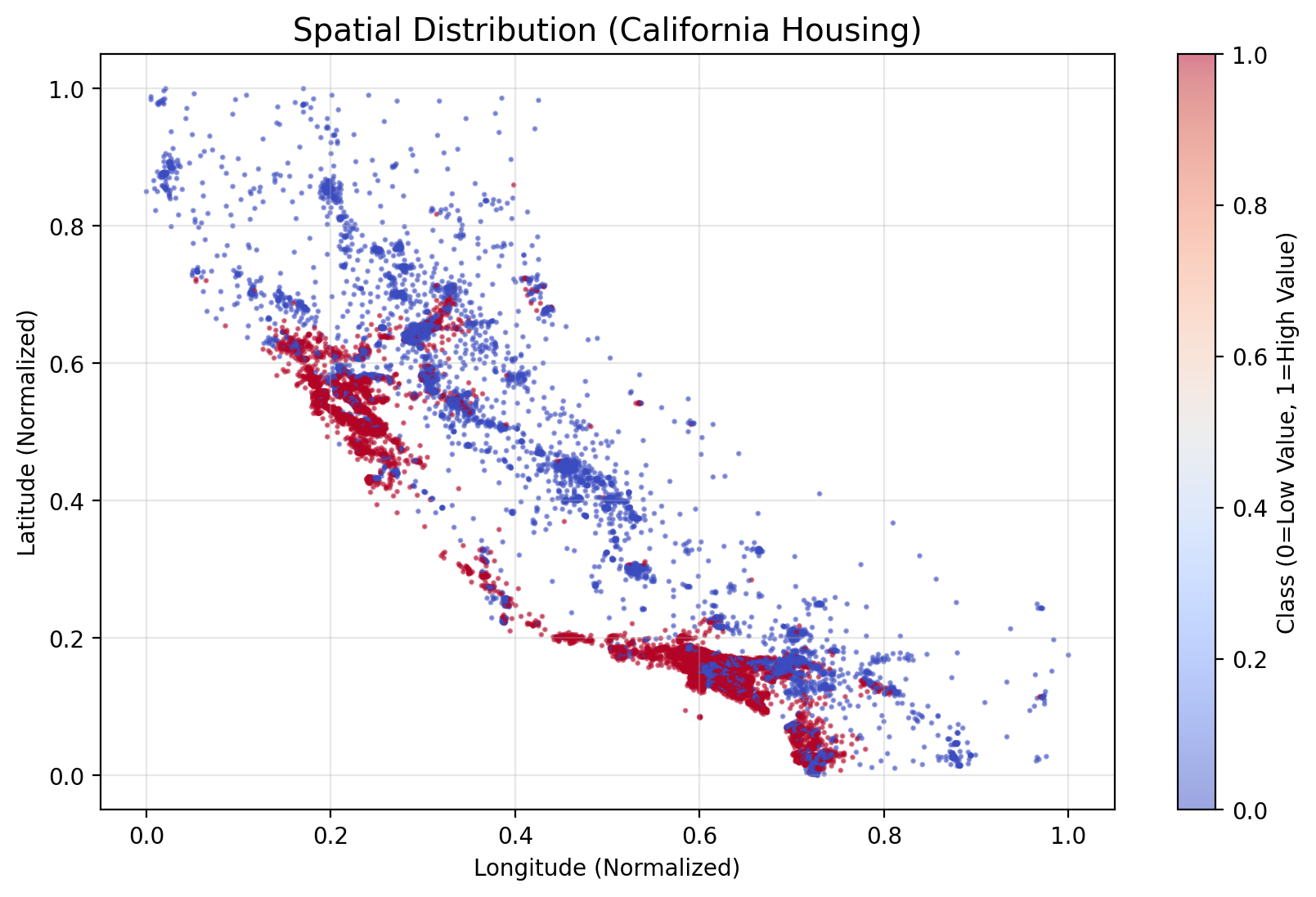}
    \caption{The plot illustrates the binary classification task: predicting whether a house value is above (red, Class 1) or below (blue, Class 0) the global median based on geospatial coordinates ($X_i \in [0, 1]^2$). }
    \label{fig:housing_data}
\end{figure}

\subsubsection{Hyperparameter Sensitivity Analysis} We keep the privacy budget fixed at $\varepsilon=1.0$ and vary $k$. For each value of $k$, we execute $N=50$ independent runs to capture the stochasticity of the mechanism.

\begin{figure}[!h]
    \centering
    \includegraphics[width=0.7\textwidth]{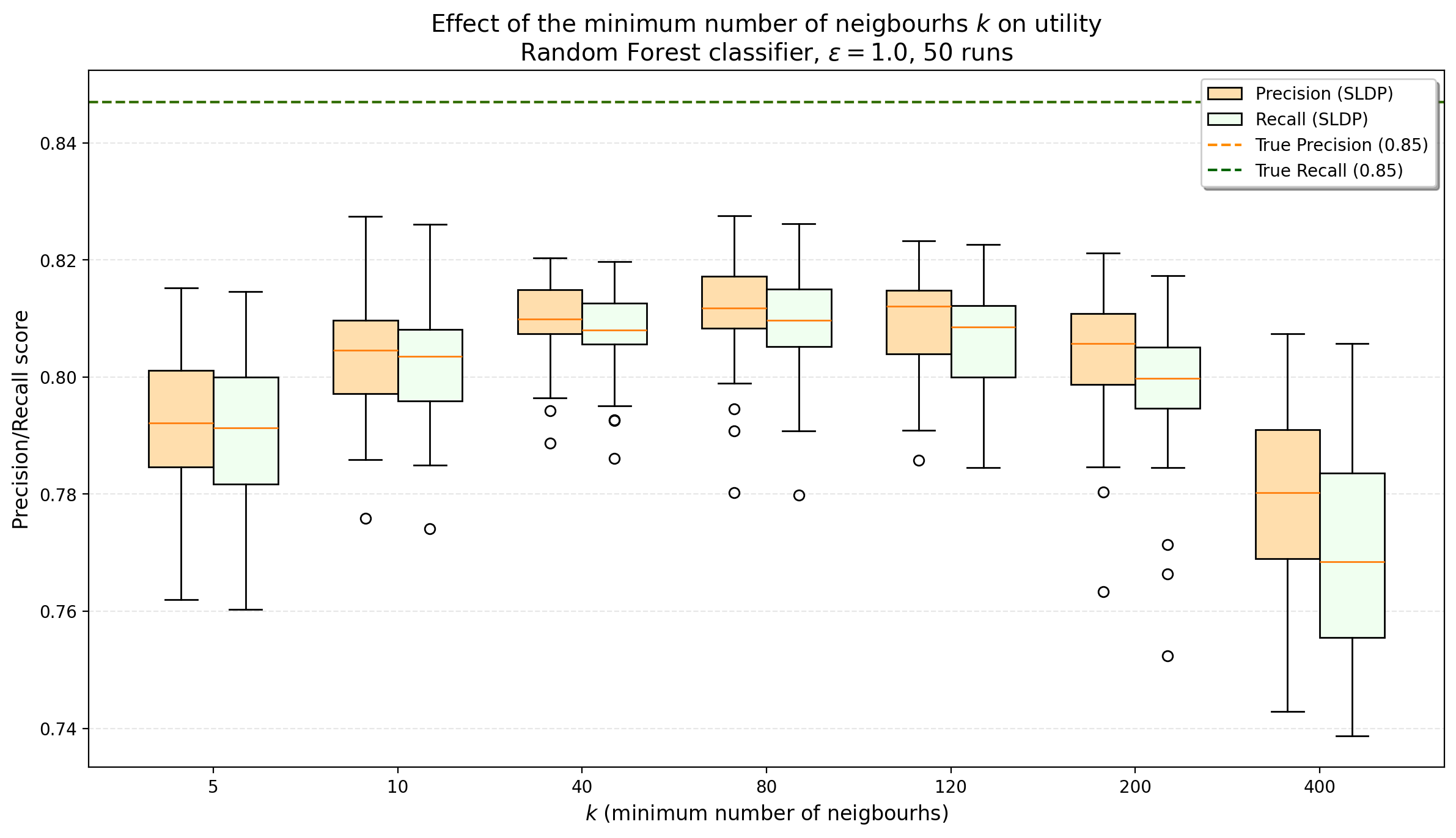}
    \caption{\textbf{Hyperparameter sensitivity analysis in SLDP.} Hyperparameter sensitivity analysis in SLDP. Distribution of Precision and Recall of a Random Forest classifier trained on SLDP-perturbed data ($\varepsilon=1.0$) as a function of the minimum number of neighbors $k$ required per privacy region.}
    \label{fig:RF_precision_recall}
\end{figure}

In experiments in Section~\ref{sec:exp_clust}, we set $k = 20$. The parameters of classifiers are as follows:
\begin{itemize}
    \item Random Forest (RF): ($N_{estimators}=50$, max depth=8).
    \item k-NN: ($k=15$, distance-weighted).
\end{itemize}

\subsection{Answering spatial queries}
\label{appx:spatial_questions}
We generate a workload of $M=200$ spatial range queries using an anchored multi-scale strategy. Unlike uniform random queries, which often result in empty regions on sparse spatial data, our approach ensures meaningful selectivity:
\begin{itemize}
    \item \textbf{Anchoring}: Each query rectangle is generated by first sampling a random data point $(c_x, c_y)$ from the dataset to serve as an anchor.
    \item \textbf{Size:}  A width $w$ and height $h$ are sampled from a log-uniform distribution (ranging from $1/256$ to $1/8$ of the domain size) to cover various spatial scales. The query rectangle is positioned such that it contains the anchor point.

    \item \textbf{Selectivity:}  We filter the generated rectangles to retain only those with a ground truth count $y_{true}$ within a specific range (e.g., $[20, 0.05N]$), discarding overly sparse or non-selective queries.
\end{itemize}

Fig.~\ref{fig:spatial_queries_example} provides a visualization of the four real-world geospatial datasets with a subsample of anchored spatial queries.
\begin{figure}[!ht]
    \centering
\includegraphics[width=\linewidth]{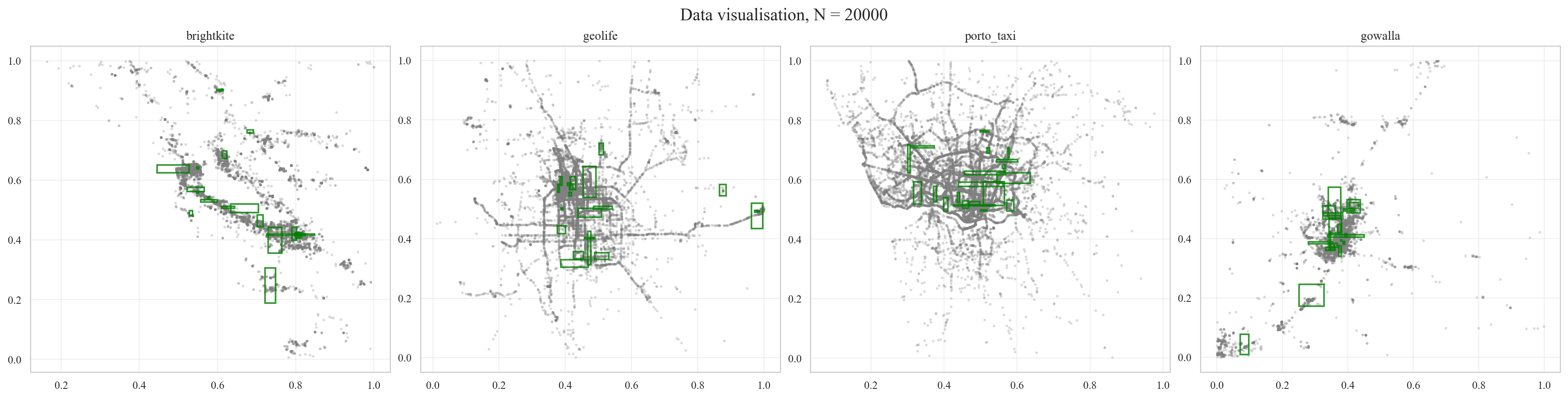}
    \caption{\textbf{Visualization of the four real-world geospatial datasets.} Each subplot displays a sample of $N=20,000$ data points from Gowalla, Brightkite, Geolife, and Porto Taxi, overlaid with representative spatial range queries (green boxes) generated by the anchored multi-scale strategy.}
    \label{fig:spatial_queries_example}
\end{figure}

\subsubsection*{Synthetic data}
To simulate spatial distributions, we generate synthetic clustered datasets within the unit square domain $[0, 1]^2$. The data generation process involves creating $N$ samples distributed around randomly positioned cluster centers, mimicking the non-uniform density often found in geospatial applications. We vary the dataset size $N \in \{2,000, \dots, 50,000\}$ to evaluate the asymptotic behavior of the algorithms. 

In this section, we generate a workload of $N_{q}$ spatial range queries using a uniform random strategy. For each query, the width $w$ and height $h$ are sampled independently from a uniform distribution $U(s_{min}, s_{max})$. The position of the rectangle (defined by its top-left corner) is then sampled uniformly from the valid domain such that the query remains fully contained within the bounds $[0, 1]^2$. Fig.~\ref{fig:spatial_queries_synthetic} visualizes a dataset with several spatial queries.

\begin{figure}[!ht]
    \centering
    \includegraphics[width=0.4\textwidth]{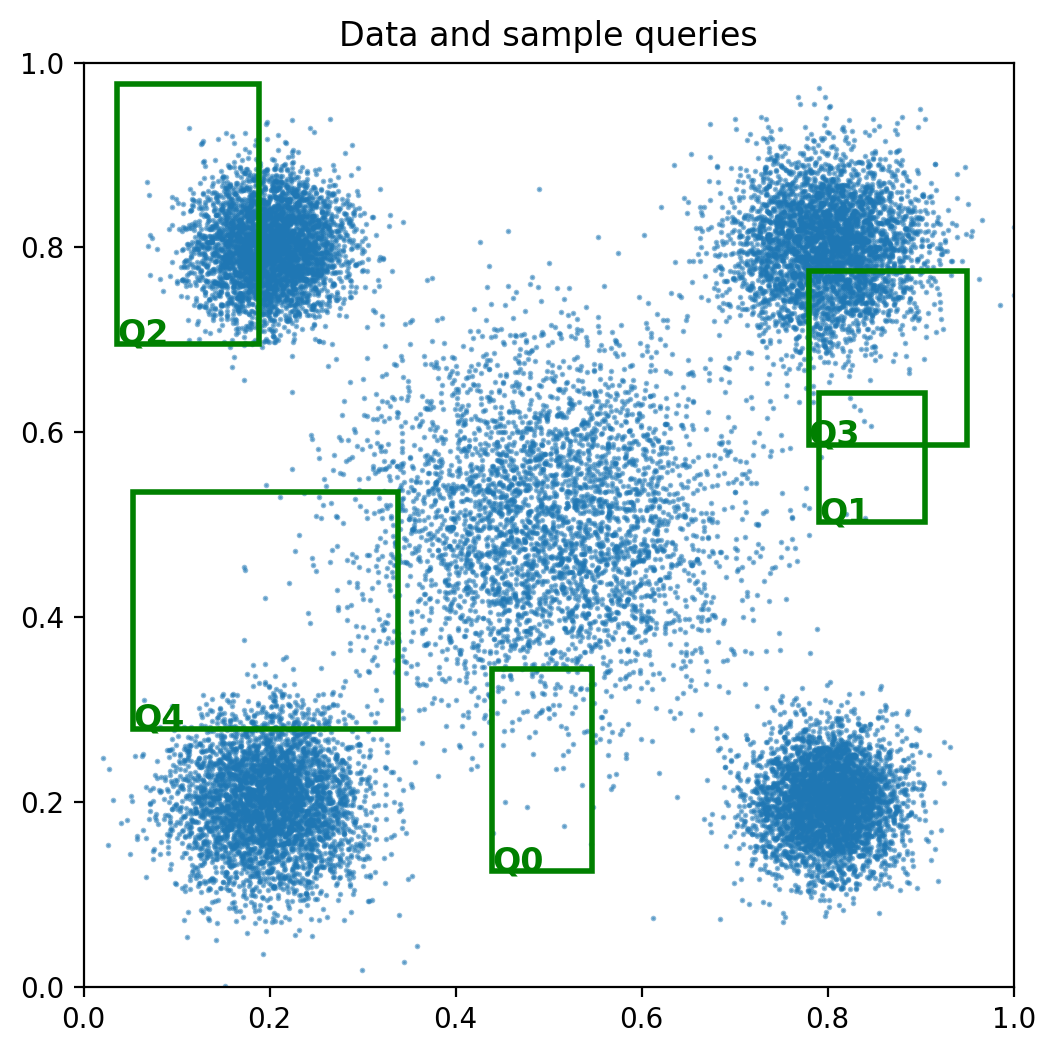}
    \caption{Data come from a Gaussian mixture. Rectangles correspond to spatial queries.}
    \label{c}
    \label{fig:spatial_queries_synthetic}
\end{figure}

Finally, we compare our SLDP method against PrivTree (central DP) \cite{zhang2016privtree} and LDP-KDTree (Local DP)~\cite{alptekin2025hierarchical}  (see Fig.~\ref{fig:comparision_synthetic}). In this setting, we use $k = 20$ for SLDP. To run Privtree, we use the parameters recommended in the original paper: 
\begin{equation}
\label{eq:param_priv_tree}
    \lambda = 7 / (3 \varepsilon_{\mathrm{tree}}), \quad \delta = \lambda \ln(4), \quad \theta=0,
\end{equation}
where $\varepsilon_{\mathrm{tree}}$ is the privacy budget spent to the tree construction.

\begin{figure}[!ht]
    \centering
    \includegraphics[width=\textwidth]{ 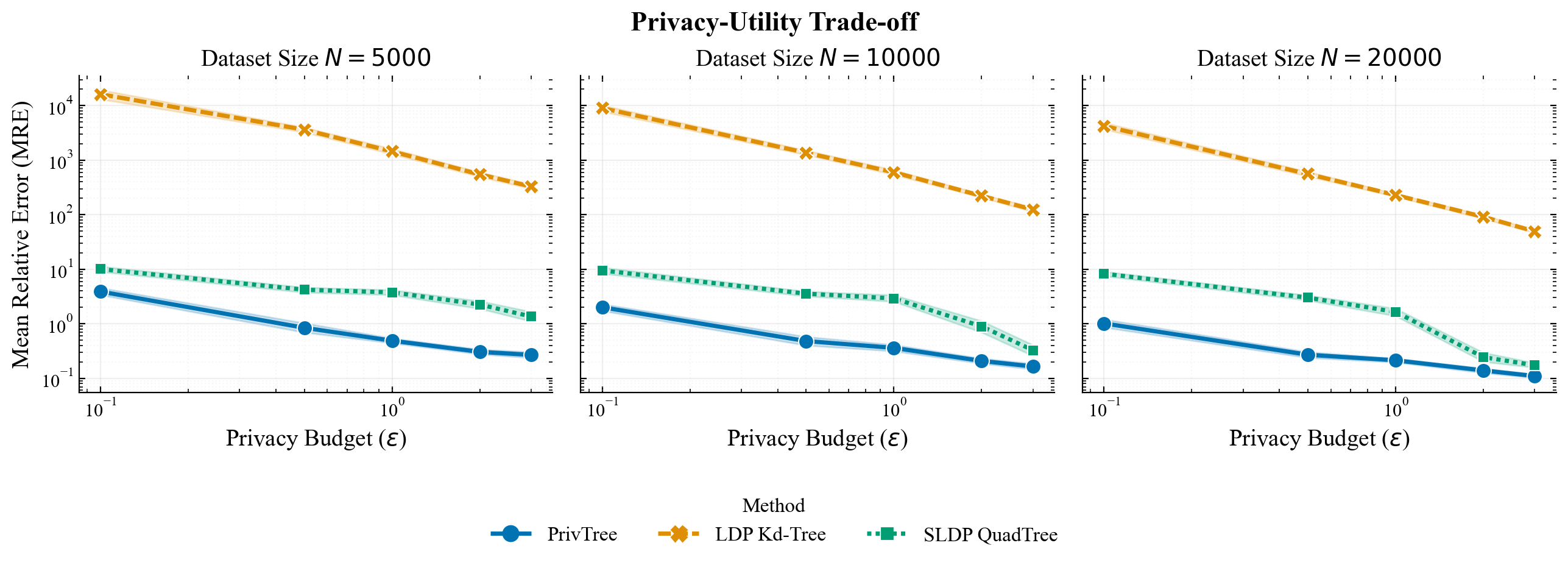}
    \caption{\textbf{Accuracy Analysis.} Evaluation of the reconstruction error (MRE) on synthetic datasets. Columns correspond to increasing dataset sizes $N$.}
    \label{c}
    \label{fig:comparision_synthetic}
\end{figure}

\end{document}